\documentclass{article} 
\usepackage{nips11submit_e,times}

\usepackage{cite}
\usepackage{times}
\usepackage{epsfig}
\usepackage{graphicx}
\usepackage{amsmath}
\usepackage{amssymb}

\usepackage{lineno}
\usepackage{color}
\usepackage{rotating}
\usepackage{amsmath,amscd,amssymb,amsfonts,amsthm}
\usepackage{algorithm, algorithmic}
\usepackage{latexsym}
\usepackage{psfrag}
\usepackage{longtable}
\usepackage{epstopdf}
\usepackage{subfigure}

\renewcommand{\mathbf}{\boldsymbol}

\newcommand{\bb}{\mathbf{b}}
\newcommand{\ee}{\mathbf{e}}

\newcommand{\g}{\mathbf{g}}

\newcommand{\xx}{\mathbf{x}}

\newcommand{\zz}{\mathbf{z}}
\newcommand{\1}{\mathbf{1}}
\newcommand{\0}{\mathbf{0}}
\newcommand{\llambda}{\mathbf{\lambda}}

\newcommand{\albf}{\mathbf{\alpha}}
\newcommand{\bebf}{\mathbf{\beta}}

\renewcommand{\Re}{{\mathbb{R}}}

\newcommand{\subjto}{\mbox{s.t.}}
\newcommand{\sign}{\mathcal{I}}

\newcommand{\ie}{{i.e., }}

\newcommand{\mypar}[1]{{\par \noindent \textbf{#1}}}

\newtheorem{theorem}{Theorem}
\newtheorem{corollary}{Corollary}

\newtheorem{remark}{Remark}

\newfloat{alg}{thp}{lop}
\floatname{alg}{Algorithm}

\usepackage[pagebackref=false,breaklinks=true,letterpaper=true,colorlinks,bookmarks=false]{hyperref}

\title{On the Lagrangian Biduality of Sparsity Minimization Problems}

\author{
Dheeraj~Singaraju, Allen~Y.~Yang, and S.~Shankar~Sastry \\
Department of EECS\\
University of California, Berkeley\\
Berkeley, CA 94720 \\
\texttt{\{dheeraj,yang,sastry\}@eecs.berkeley.edu} \\
\And
Ehsan~Elhamifar and Roberto~Tron\\
Center for Imaging Science\\
Johns Hopkins University \\
Baltimore, MD 21218 \\
\texttt{\{ehsan,tron\}@cis.jhu.edu}
}

\nipsfinalcopy 

\begin{document}

\maketitle

\begin{abstract}
Recent results in Compressive Sensing have shown that, under certain conditions, the solution to an underdetermined system of linear equations with sparsity-based regularization can be accurately recovered by solving convex relaxations of the original problem. In this work, we present a novel primal-dual analysis on a class of sparsity minimization problems. We show that the Lagrangian bidual (i.e., the Lagrangian dual of the Lagrangian dual) of the sparsity minimization problems can be used to derive interesting convex relaxations: the bidual of the $\ell_0$-minimization problem is the $\ell_1$-minimization problem; and the bidual of the $\ell_{0,1}$-minimization problem for enforcing group sparsity on structured data is the $\ell_{1,\infty}$-minimization problem. The analysis provides a means to compute per-instance non-trivial lower bounds on the (group) sparsity of the desired solutions. In a real-world application, the bidual relaxation  improves the performance of a sparsity-based classification framework applied to robust face recognition.
\end{abstract}

\section{Introduction}
The last decade has seen a renewed interest in the problem of solving an underdetermined system of equations $A\xx = \bb$, $A \in \Re^{m \times n}, \bb \in \Re^m$, where $m << n$, by regularizing its solution to be sparse, \ie having very few non-zero entries. Specifically, if one aims to find $\xx$ with the least number of nonzero entries that solves the linear system, the problem is known as $\ell_0$-minimization:
\begin{equation}
\label{eq:entry}
(P_0): \quad \xx_0^* = \underset{\xx \in \Re^n}{\arg\!\min}~ \|\xx\|_0\quad \subjto \quad A\xx = \bb.
\end{equation}

The problem $(P_0)$ is intended to seek \emph{entry-wise sparsity} in $\xx$ and is known to be NP-hard in general. In Compressive Sensing (CS) literature, it has been shown that the solution to \eqref{eq:entry} often can be obtained by solving a more tractable linear program, namely, $\ell_1$-minimization \cite{DonohoD2003,CandesE2006}:
\begin{equation}
\label{eq:l1min}
(P_1): \quad \xx_1^* = \underset{\xx \in \Re^n}{\arg\!\min}~ \|\xx\|_1\quad \subjto \quad A\xx = \bb.
\end{equation}
This unconventional equivalence relation between $(P_0)$ and $(P_1)$ and the more recent numerical solutions \cite{BrucksteinA2009,LorisI2009} to efficiently recover high-dimensional sparse signal have been a very competitive research area in CS. Its broad applications have included sparse error correction \cite{CandesE2005}, compressive imaging \cite{WakinM2006}, image denoising and restoration \cite{ProtterM2009,MairalJ2008}, and face recognition \cite{WrightJ2009,ElhamifarE2011}, to name a few.

In addition to enforcing entry-wise sparsity in a linear system of equations, the notion of \emph{group sparsity} has attracted increasing attention in recent years \cite{StojnicM2009,EldarY2009,ElhamifarE2011}. In this case, one assumes that the matrix $A$ has some underlying structure, and can be grouped into blocks: $A = \begin{bmatrix} A_1 & \cdots & A_K \end{bmatrix}$, where $A_k \in \Re^{m \times d_k}$ and $\sum_{k=1}^K d_k = n$. Accordingly, the vector $\xx$ is split into several blocks as $\xx^\top = \begin{bmatrix} \xx_1^\top & \ldots   & \xx_K^\top\end{bmatrix}$, where $\xx_k \in \Re^{d_k}$. In this case, it is of interest to estimate $\xx$ with the least number of blocks containing non-zero entries.
The group sparsity minimization problem is posed as
\begin{equation}
\begin{split}
(P_{0,p}):\quad \xx_{0,p}^*  & = \underset{\xx}{\arg\!\min} \sum_{k=1}^K \sign(\|\xx_k\|_p>0), \quad \subjto \quad A\xx \doteq  \begin{bmatrix} A_1 & \cdots & A_K \end{bmatrix} \begin{bmatrix} \xx_1 \\ \vdots \\ \xx_K \end{bmatrix} = \bb,
\end{split}
\label{eq:group}
\end{equation}
where $\sign(\cdot)\in\Re$ is the indicator function. Since the expression $\sum_{k=1}^K \sign(\|\xx_k\|_p>0)$ can be written as $\|\begin{bmatrix} \|\xx_1\|_p & \cdots & \|\xx_K\|_p \end{bmatrix}\|_0$, it is also denoted as $\ell_{0,p}(\xx)$, the $\ell_{0,p}$-norm of $\xx$.


Enforcing group sparsity exploits the problem's underlying structure and can improve the solution's interpretability.   For example, in a sparsity-based classification (SBC) framework applied to face recognition, the columns of $A$ are vectorized training images of human faces that can be naturally grouped into blocks corresponding to different subject classes, $\bb$ is a vectorized query image, and the entries in $\xx$ represent  the coefficients of linear combination of all the training images for reconstructing $\bb$. Group sparsity lends itself naturally to this problem since  it is desirable to use images of the smallest number of subject classes to reconstruct and subsequently classify a  query image.

Furthermore, the problem of robust face recognition has considered an interesting modification known as the \emph{cross-and-bouquet} (CAB) model: $\bb=A\xx+\ee$, where $\ee \in \Re^m$ represents possible sparse error corruption on the observation $\bb$ \cite{WrightJ2008}. It can be argued that the CAB model  can be solved as a group sparsity problem in \eqref{eq:group}, where the coefficients of $\ee$ would be the $(K+1)^\text{th}$ group.  However,  this problem has a trivial solution for $\ee = \bb$ and $\xx = \0$, which would have the smallest possible group sparsity. Hence, it is necessary to further regularize the entry-wise sparsity in $\ee$.

To this effect, one considers a mixture of the previous two cases, where one aims to enforce entry-wise sparsity as well as group sparsity such that $\xx$ has very few number of non-zero blocks \emph{and} the reconstruction error $\ee$ is also sparse. The \emph{mixed sparsity} minimization problem can be posed as
\begin{equation}
\begin{split}
(MP_{0,p}):  ~~ \{\xx^*_{0,p}, \ee^*_0\}  & = \underset{(\xx,\ee)}{\arg\!\min}~\ell_{0,p}(\xx) + \gamma\|\ee\|_0, ~~ \subjto ~~ \begin{bmatrix} A_1 & \cdots & A_K \end{bmatrix} \begin{bmatrix} \xx_1 \\ \vdots \\ \xx_K \end{bmatrix} = \bb+\ee,
\end{split}
\label{eq:mixed}
\end{equation}
where $\gamma \ge 0$ controls the tradeoff between the entry-wise sparsity and group sparsity. 

Due to the use of the counting norm, the optimization problems in
\eqref{eq:group} and \eqref{eq:mixed} are also NP-hard in
general. Hence, several recent works have focused on developing
tractable convex relaxations for these problems. In the case of group
sparsity, the relaxation involves replacing the $\ell_{0,p}$-norm with
the $\ell_{1,p}$-norm, where $\ell_{1,p}(\xx)\doteq
\bigl\|\begin{bmatrix} \|\xx_1\|_p & \cdots &
  \|\xx_K\|_p \end{bmatrix}\bigr\|_1 = \sum_{k=1}^K \|\xx_k\|_p$.
These relaxations are also used for the mixed sparsity case \cite{ElhamifarE2011}. 

In this work, we are interested in deriving and analyzing convex relaxations for general sparsity minimization problems. In the entry-wise case, the main theoretical understanding of the link between the original NP-hard problem in \eqref{eq:entry} and its convex relaxation has been given by the simple fact that the $\ell_1$-norm is a convex surrogate of the $\ell_0$-norm. However, in the group sparsity case, a similar relaxation produces a family of convex surrogates, \ie  $\ell_{1,p}(\xx)$, whose value depends on  $p$. This raises the question whether there is a preferable value of $p$ for the relaxation of the group sparsity minimization problem? In fact, we consider the following more important question:

\noindent \emph{Is there a unified framework for deriving convex relaxations of general sparsity recovery problems?}




\subsection{Paper contributions}
We present a new optimization-theoretic framework based on Lagrangian duality for deriving convex relaxations of sparsity minimization problems. Specifically, we introduce a new class of equivalent optimization problems for $(P_0)$, $(P_{0,p})$ and $(MP_{0,p})$, and derive the Lagrangian duals of the original NP-hard problems. We then consider the Lagrangian dual of the Lagrangian dual to get a new optimization problem that we term as the \emph{Lagrangian bidual} of the primal problem. We show that the Lagrangian biduals are convex relaxations of the original sparsity minimization problems. Importantly, we show that the Lagrangian biduals for the $(P_0)$ and $(P_{0,p})$ problems correspond to minimizing the $\ell_1$-norm  and the $\ell_{1,\infty}$-norm, respectively.

Since the Lagrangian duals for  $(P_0)$, $(P_{0,p})$ and
$(MP_{0,p})$ are linear programs, there is no duality gap
between the Lagrangian duals and the corresponding Lagrangian
biduals. Therefore, the  bidual based convex relaxations can be
interpreted as maximizing the Lagrangian duals of the original
sparsity minimization problems. This provides new interpretations for
the relaxations of sparsity minimization problems. Moreover, since the
Lagrangian dual of a minimization problem provides a lower bound for
the optimal value of the primal problem, we show that the optimal objective value of the convex
relaxation provides a non-trivial lower bound on the sparsity of the  true solution to the primal problem. 

\section{Lagrangian biduals for sparsity minimization problems}
In what follows, we will derive the Lagrangian bidual for the mixed sparsity minimization problem, which generalizes the entry-wise sparsity and group sparsity cases (also see Section \ref{sec:results}). Specifically, we will derive the Lagrangian bidual for the following optimization problem:
\begin{equation}
\begin{split}
\xx^*  & = \underset{\xx}{\arg\!\min}~ \sum_{k=1}^K \big[\alpha_k \sign(\|\xx_k\|_p>0) + \beta_k \| \xx_k \|_0 \big] , ~~ \subjto ~~ \begin{bmatrix} A_1 & \cdots & A_K \end{bmatrix} \begin{bmatrix} \xx_1 \\ \vdots \\ \xx_K \end{bmatrix} = \bb,
\end{split}
\label{eq:primal0}
\end{equation}
where $\forall k = 1,\ldots,K: \alpha_k \ge 0$ and $\beta_k \ge 0$. Given any unique, finite solution $\xx^*$ to \eqref{eq:primal0}, there exists a constant $M>0$ such that the absolute values of the entries of $\xx^*$ are less than $M$, namely, $\|\xx^*\|_\infty \le M$. Note that if \eqref{eq:primal0} does not have a unique solution, it might not
be possible to choose a finite-valued $M$ that upper bounds all the solutions. In this case, a finite-valued $M$ may be viewed as a regularization term for the desired solution. To this effect, we consider the following modified version of \eqref{eq:primal0} where we introduce the box constraint that $\|\xx\|_\infty \le M$:
\begin{equation}
\begin{split}
\xx_\text{primal}^*  & = \underset{\xx}{\arg\!\min}~ \sum_{k=1}^K \big[\alpha_k \sign(\|\xx_k\|_p>0) + \beta_k \| \xx_k \|_0 \big] , ~~ \subjto ~~ A\xx = \bb \text{ and } \|\xx\|_\infty \le M,
\end{split}
\label{eq:primal1}
\end{equation}
where $M$ is chosen as described above to ensure that the optimal values of \eqref{eq:primal1} and \eqref{eq:primal0}  are the same.

\mypar{Primal problem.} We will now frame an equivalent optimization problem for \eqref{eq:primal1}, for which we introduce some new notation. Let $\zz \in \{0,1\}^n$ be an entry-based sparsity indicator for $\xx$, namely, $z_i =0$ if $x_i =0$ and $z_i=1$ otherwise. We also introduce a group-based sparsity indicator vector $\g \in \{0,1\}^K$, whose $k^\text{th}$ entry $g_k$  denotes whether the $k^\text{th}$ block $\xx_k$ contains non-zero entries or not, namely, $g_k = 0$ if $\xx_k = \0$ and $g_k=1$ otherwise.  To express this constraint, we introduce a matrix $\Pi \in \{0,1\}^{n \times K}$, such that $\Pi_{i,j} = 1$ if the $i^\text{th}$ entry of $\xx$ belongs to the $j^\text{th}$ block and $\Pi_{i,j} = 0$ otherwise.  Finally,  we denote the positive component and negative component of $\xx$ as $\xx_+\ge 0$ and $\xx_-\ge 0$, respectively, such that $\xx = \xx_+ - \xx_-$. 

Given these definitions, we see that \eqref{eq:primal1} can be reformulated as
\begin{equation}
\label{eq:primal}
\begin{split}
& \! \{\xx_+^*,\xx_-^*,\zz^*,\g^*\}  \!=\!\!   \underset{\{\xx_+,\xx_-,\zz,\g\}}{\arg\!\min} \big[\albf^\top \g + \bebf^\top\zz\big] , ~ \subjto  \text{ (a) } \xx_+ \ge 0, \text{ (b) } \xx_- \ge 0, \text{ (c) } \g \in \{0,1\}^K, \!\!\\
&  \text{(d) } \zz \in \{0,1\}^n  \text{ (e) } A(\xx_+ - \xx_-) = \bb,  \text{ (f) } \Pi\g \ge \frac{1}{M}(\xx_+ + \xx_-),  
\text{ and (g) }  \zz \ge \frac{1}{M}(\xx_+ + \xx_-),
\end{split}
\end{equation}
where $\albf = \begin{bmatrix} \alpha_1& \cdots & \alpha_k \end{bmatrix}^\top \in \Re^k$ and $\bebf = [ \cdots ~~  \underbrace{\beta_k  \cdots \beta_k}_{d_k \text{times}} ~~ \cdots]^\top \in \Re^n$. 

Constraints (a)--(d) are used to enforce the aforementioned conditions on the values of the solution. While constraint (e) enforces the condition that the original system of linear equations is satisfied, the constraints (f) and (g) ensure that the group sparsity indicator $\g$ and the entry-wise sparsity indicator $\zz$ are consistent with the entries of $\xx$.

\medskip \mypar{Lagrangial dual.} The Lagrangian function for \eqref{eq:primal} is given as
\begin{equation}
\begin{split}
L(\xx_+,\xx_-,\zz,\g, \llambda_1,\llambda_2,\llambda_3,\llambda_4,\llambda_5)  & = \albf^\top \g + \bebf^\top\zz - \llambda_1^\top\xx_+ - \llambda_2^\top\xx_-  + \llambda_3^\top(\bb-A\xx_+ + A\xx_-)  \\&  + \llambda_4^\top(\frac{1}{M}(\xx_+ + \xx_-) - \Pi\g)  + \llambda_5^\top(\frac{1}{M}(\xx_+ + \xx_-) - \zz) , \\
\end{split}
\end{equation}
where $\llambda_1 \ge \0$, $\llambda_2\ge \0$, $\llambda_4 \ge \0$, and
$\llambda_5 \ge \0$.  In order to obtain the Lagrangian dual function,
we need to minimize $L(\cdot)$ with respect to $\xx_+$, $\xx_-$, $\g$ and $\zz$\cite{BoydS2004}.  Notice that if the coefficients of
$\xx_+$ and $\xx_-$, \ie $\frac{1}{M}(\llambda_4+\llambda_5)-
A^\top\llambda_3 -\llambda_1$ and $\frac{1}{M}(\llambda_4+\llambda_5)+
A^\top\llambda_3 -\llambda_2$ are non-zero, the minimization of $L(\cdot)$ with respect to $\xx_+$ and $\xx_-$  is
unbounded below. To this effect, the constraints that these coefficients are equal to $0$ form constraints on the dual variables.  Next, consider the minimization of $L(\cdot)$ with respect to $\g$.  Since each entry $g_k$ only takes values $0$ or $1$, its optimal value $\hat{g}_k$ that minimizes $L(\cdot)$ is given as
\begin{equation}
\begin{split}
\hat{g}_k = \begin{cases} 0 & \text{if } \alpha_k - (\Pi^\top\llambda_4)_k > 0, \text{ and} \\
1 & \text{ otherwise.} \end{cases}
\end{split}
\end{equation}
A similar expression can be computed for the minimization with respect
to $\zz$. As a consequence, the Lagrangian dual problem can be derived as
\begin{equation}
\label{eq:dual0}
\begin{split}
 \{\llambda^*_i\}_{i=1}^5 & = \underset{ \{\llambda_i\}_{i=1}^5}{\arg\!\max} \big[\llambda_3^\top\bb + \1^\top\min\{\0,\albf-\Pi^\top\llambda_4\} +  \1^\top\min\{\0,\bebf-\llambda_5\}\big], \subjto \\
& \text{(a) } \forall i=1,2,4,5: \llambda_i \ge \0, \text{ (b) } \frac{1}{M}(\llambda_4+\llambda_5)- A^\top\llambda_3 -\llambda_1 = \0 \\ & \text{and (c) } \frac{1}{M}(\llambda_4+\llambda_5)+ A^\top\llambda_3 -\llambda_2 = \0.
\end{split}
\end{equation}

This can be further simplified by rewriting it as the following linear program:
\begin{equation}
\label{eq:dual}
\begin{split}
& \!\!\!\!\{\llambda^*_i\}_{i=3}^7  = \underset{ \{\llambda_i\}_{i=3}^7}{\arg\!\max} \big[\llambda_3^\top\bb + \1^\top\llambda_6 +  \1^\top\llambda_7\big], \subjto~ \text{(a) } \llambda_4 \ge \0,  \text{ (b) } \llambda_5 \ge \0, \text{ (c) } \llambda_6 \le 0, \text{ (d) } \llambda_7 \le 0, \!\!\!\!\!\!\!\!\!\!\!\!\\
&\!\!\!\!\text{ (e) } \llambda_6 \le \albf-\Pi^\top\llambda_4, \text{ (f) } \llambda_7 \le \bebf-\llambda_5 \text{ and (g) }-\frac{1}{M}(\llambda_4+\llambda_5) \le A^\top\llambda_3 \le \frac{1}{M}(\llambda_4+\llambda_5).  
\end{split}
\end{equation}

Notice that we have made two changes in going from \eqref{eq:dual0} to \eqref{eq:dual}. First, we have replaced constraints (b) and (c) in \eqref{eq:dual0} with the constraint (g) in \eqref{eq:dual} and eliminated $\llambda_1$ and $\llambda_2$ from \eqref{eq:dual}. Second, we have introduced variables $\llambda_6$ and $\llambda_7$ to encode the ``min" operator in the objective function of \eqref{eq:dual0}. 

\mypar{Lagrangian bidual.} We will now consider the Lagrangian dual of \eqref{eq:dual}, which will be referred to as the \emph{Lagrangian bidual} of \eqref{eq:primal}. It can be verified that the Lagrangian dual of \eqref{eq:dual} is given as
\begin{equation}
\label{eq:bidual0}
\begin{split}
& \!\!\!\!\!\! \{\xx_+^*,\xx_-^*,\zz^*,\g^*\}  \!=\!\!   \underset{\{\xx_+,\xx_-,\zz,\g\}}{\arg\!\min}  \albf^\top\g + \bebf^\top\zz
\quad \subjto  \text{ (a) } \xx_+ \ge 0, \text{ (b) } \xx_- \ge 0, \text{ (c) } \g \in [0,1]^K, \!\!\\
& \!\!\!\!\!\! \text{(d) } \zz \in [0,1]^n  \text{ (e) } A(\xx_+ - \xx_-) = \bb,  \text{ (f) } \Pi\g \ge \frac{1}{M}(\xx_+ + \xx_-)  
\text{ and (g) }  \zz \ge \frac{1}{M}(\xx_+ + \xx_-).
\end{split}
\end{equation}

Notice that in going from \eqref{eq:primal} to \eqref{eq:bidual0}, the discrete valued variables $\zz$ and $\g$ have been relaxed to take real values between $0$ and $1$. Given that $\zz \le \1$ and noting that $\xx$ can be represented as $\xx =\xx_+ - \xx_-$, we can conclude from constraint (g) in \eqref{eq:bidual0} that the solution $\xx^*$ satisfies $\|\xx^*\|_\infty \le M$. Moreover, given that $\g$ and $\zz$ are relaxed to take real values, we see that the optimal values for $g^*_k$ and $z^*_i$  are $\frac{1}{M}\|\xx^*_k\|_\infty$ and $\frac{1}{M}|x^*_i|$, respectively. Hence, we can eliminate constraints (f) and (g) by replacing $\zz$ and $\g$ by these optimal values. It can then be verified that solving \eqref{eq:bidual0} is equivalent to solving the problem:
\begin{equation}
\label{eq:bidual}
\xx_\text{bidual}^* =  \underset{\xx}{\arg\!\min}  \frac{1}{M} \sum_{k=1}^K \big[\alpha_k \|\xx_k\|_\infty + \beta_k \| \xx_k \|_1 \big] 
\quad \subjto  \text{ (a) }  A\xx = \bb  \text{ and (b) } \|\xx\|_\infty \le M.
\end{equation}

This is the Lagrangian bidual for \eqref{eq:primal}.

\section{Theoretical results from the biduality framework}
\label{sec:results}
In this section, we first describe some properties of the biduality framework in general. We will then focus on some important results for the special cases of entry-wise sparsity and group sparsity.

\begin{theorem}
The optimal value of the Lagrangian bidual in \eqref{eq:bidual} is a lower bound on the optimal value of the NP-hard primal problem in \eqref{eq:primal}.
\label{thm:optval}
\end{theorem}
\begin{proof}
Since there is no duality gap between a linear program and its Lagrangian dual \cite{BoydS2004}, the optimal values of  the Lagrangian dual in \eqref{eq:dual}  and the Lagrangian bidual in \eqref{eq:bidual} are the same. Moreover, we know that the optimal value of a primal minimization problem is always bounded below by the optimal value of its Lagrangian dual \cite{BoydS2004}. We hence have the required result.
\end{proof}

\begin{remark} 
Since the original primal problem in \eqref{eq:primal} is NP-hard, we note that the duality gap between the primal  and its dual in \eqref{eq:dual} is non-zero in general. Moreover, we notice that as we increase $M$ (\ie a  more conservative estimate),  the optimal value of the primal is unchanged, but the optimal value of the bidual in \eqref{eq:bidual} decreases.  Hence, the duality gap increases as $M$ increases.
\end{remark}


$M$ in \eqref{eq:primal1} should preferably be equal to $\|\xx^*_\text{primal}\|_\infty$, which may not be possible to estimate accurately in practice. Therefore, it is of interest to analyze the effect of taking a very conservative estimate of $M$, \ie choosing a large value for $M$.  In what follows, we show that taking a conservative estimate of $M$ is equivalent to dropping the box constraint in the bidual. 

For this purpose, consider the following modification of the bidual:
\begin{equation}
\label{eq:bidual-unconst}
\xx_\text{bidual-conservative}^* =  \underset{\xx}{\arg\!\min} \sum_{k=1}^K \big[\alpha_k \|\xx_k\|_\infty + \beta_k \| \xx_k \|_1 \big]  \quad \subjto \quad  A\xx = \bb,
\end{equation}
where we have essentially dropped the box constraint (b) in \eqref{eq:bidual}. It is easy to verify that $\forall M \ge \max\{\|\xx_\text{primal}^*\|_\infty,\|\xx_\text{bidual-conservative}^*\|_\infty\}$,  we have that $\xx_\text{bidual}^* = \xx_\text{bidual-conservative}^*$. Therefore, we see that taking a conservative value of $M$ is equivalent to solving the modified bidual in \eqref{eq:bidual-unconst}.

\subsection{Results for entry-wise sparsity minimization} 
\label{sec:entry-results}
Notice that by substituting $\alpha_1= \cdots = \alpha_K =0$ and $\beta_1 = \cdots = \beta_K = 1$, the optimization problem in \eqref{eq:primal0}  reduces to the entry-wise sparsity minimization problem in \eqref{eq:entry}. Hence, the Lagrangian bidual to the $M$-regularized entry-wise sparsity problem $(P_0)$ is:
\begin{equation}
\label{eq:entry-bidual}
\xx_\text{entry-wise-bidual}^* =  \underset{\xx}{\arg\!\min}  \frac{1}{M}  \| \xx \|_1
\quad \subjto  \text{ (a) }  A\xx = \bb  \text{ and (b) } \|\xx\|_\infty \le M.
\end{equation}

More importantly, we can also conclude from \eqref{eq:bidual-unconst} that solving the Lagrangian bidual to the entry-wise sparsity problem with a conservative estimate of $M$ is equivalent to solving the problem:
\begin{equation}
\label{eq:entry-bidual-unconst}
\xx_\text{entry-wise-bidual-conservative}^* =  \underset{\xx}{\arg\!\min}   \| \xx \|_1
\quad \subjto \quad  A\xx = \bb,
\end{equation}
which is precisely the well-known $\ell_1$-norm relaxation for $(P_0)$. 
Our framework therefore provides a new interpretation for this relaxation:
\begin{remark}
The $\ell_1$-norm minimization problem in \eqref{eq:entry-bidual-unconst} is the Lagrangian bidual of the $\ell_0$-norm minimization problem in \eqref{eq:entry}, and solving \eqref{eq:entry-bidual-unconst} is equivalent to maximizing the dual of \eqref{eq:entry}. 
\end{remark} 
We further note that we can now use the solution of \eqref{eq:entry-bidual} to derive a non-trivial lower bound for the primal objective function which is precisely the sparsity of the desired solution. More specifically, we can use Theorem \ref{thm:optval} to conclude the following result:

\begin{corollary} \label{corr:l0l1} Let $\xx_0^*$ be the solution to \eqref{eq:entry}. We have that $\forall M \ge \|\xx_0^*\|_\infty$, the sparsity of $\xx_0^*$, \ie $\|\xx_0^*\|_0$ is bounded below by $ \frac{1}{M}  \| \xx_\text{entry-wise-bidual}^* \|_1 $.
\end{corollary}

Due to the non-zero duality gap in the primal entry-wise sparsity minimization problem, the above lower bound provided by Corollary \ref{corr:l0l1} is not tight in general. 



\subsection{Results for group sparsity minimization} 
\label{sec:group-results}
Notice that by substituting $\alpha_1= \cdots = \alpha_K =1$ and
$\beta_1 = \cdots = \beta_K = 0$, the optimization problem in
\eqref{eq:primal0}  reduces to the group sparsity minimization problem
in \eqref{eq:group}. Hence, the Lagrangian bidual of the group sparsity problem is:
\begin{equation}
\label{eq:group-bidual}
\xx_\text{group-bidual}^* =  \underset{\xx}{\arg\!\min}  \frac{1}{M} \sum_{k=1}^K \|\xx_k\|_\infty
\quad \subjto \quad \text{ (a) }  A\xx = \bb  \text{ and (b) } \|\xx\|_\infty \le M.
\end{equation}

As in the case of entry-wise sparsity above, solving the bidual to the group sparsity problem with a conservative estimate of $M$ is equivalent to solving:
\begin{equation}
\label{eq:group-bidual-unconst}
\xx_\text{group-bidual-conservative}^* =  \underset{\xx}{\arg\!\min}  \sum_{k=1}^K \|\xx_k\|_\infty
\quad \subjto \quad  A\xx = \bb,
\end{equation}
which is the convex $\ell_{1,\infty}$-norm relaxation of the $\ell_{0,p}$-min problem \eqref{eq:group}. 
In other words, the biduality framework selects the $\ell_{1,\infty}$-norm out of the entire family of  $\ell_{1,p}$-norms as the convex surrogate of the $\ell_{0,p}$-norm.


Finally, we use Theorem \ref{thm:optval} to show that the solution obtained by minimizing the $\ell_{1,\infty}$-norm provides a lower bound for the group sparsity. 


\begin{corollary} \label{corr:grp} Let $\xx^*_{0,p}$ be the solution to \eqref{eq:group}. For any $M \ge \|\xx^*_{0,p}\|_\infty$, the group sparsity of $\xx^*_{0,p}$, \ie $\ell_{0,p}(\xx^*_{0,p}) $, is bounded below by $ \frac{1}{M} \ell_{1,\infty}(\xx_\text{group-bidual}^*) $.
\end{corollary}

The $\ell_{1,\infty}$-norm seems to be an interesting choice for computing the lower bound of the group sparsity, as compared to other $\ell_{1,p}$-norms for finite $p < \infty$. For example, consider the case when $p=1$, where the $\ell_{1,p}$-norm is equivalent to the $\ell_1$-norm. Assume that $A$ consists of a single block with several columns so that the maximum number of non-zero blocks is $1$. Denote the solution to the $\ell_1$-minimization problem as $\xx_1^*$. It is possible to construct examples  (also see Figure \ref{fig:entry}) where $\frac{1}{M}\ell_{1,1}(\xx_1^*)  = \frac{1}{M}\ell_{1}(\xx_1^*) > 1 $. Hence, it is  unclear in general if the solutions obtained by minimizing $\ell_{1,p}$-norms for finite-valued $p < \infty$ can help provide lower bounds for the group sparsity.


\section{Experiments}
\label{sec:exp}
We now present experiments to evaluate the bidual framework for minimizing entry-wise sparsity  and mixed sparsity. We present experiments on synthetic data to show that our framework can be used to compute non-trivial lower bounds for the entry-wise sparsity minimization problem. We then consider the face recognition problem where we compare the performance of the bidual-based $\ell_{1,\infty}$-norm relaxation  with that of  the $\ell_{1,2}$-norm relaxation for mixed sparsity minimization. 

We use boxplots to provide a concise representation of our results' statistics. The top and bottom edge of a boxplot for a set of values indicates the maximum and minimum of the values. The bottom and top extents of the box indicate the $25$ and $75$ percentile mark. The red mark in the box indicates the median and the red crosses outside the boxes indicate potential outliers.

\mypar{Entry-wise sparsity.} We now explore the practical implications of Corollary~\ref{corr:l0l1} through synthetic experiments. We randomly generate entries of $A\in \Re^{128 \times 256}$ and $\xx_0\in\Re^{256}$ from a Gaussian distribution with unit variance. The sparsity of $\xx_0$ is varied from $1$ to $64$ in steps of $3$. We solve \eqref{eq:entry-bidual} with $\bb=A\xx_0$ using  $M=M_0,  2M_0$ and $5M_0$, where $M_0=\|\xx_0\|_\infty$. We use Corollary \ref{corr:l0l1} to compute lower bounds on the true sparsity, \ie $\|\xx_0\|_0$. We repeat this experiment $1000$ times for each sparsity level and Figure~\ref{fig:entry} shows the boxplots for the bounds computed from these experiments.

We first analyze the lower bounds computed when $M=M_0$, in Figure \ref{fig:m0}. As explained  in Section \ref{sec:entry-results}, the bounds are not expected to be tight due to the duality gap. Notice that for extremely sparse solutions, the maximum of the computed bounds is close to the true sparsity but this diverges as the sparsity of $\xx_0$ reduces. The median value of the bounds is much looser and we see that the median also diverges as the sparsity of $\xx_0$ reduces. Furthermore, the computed lower bounds seem to grow linearly as a function of the true sparsity. Similar trends are observed for $M=2M_0$ and $5M_0$ in Figures \ref{fig:2m0} and \ref{fig:5m0}, respectively. As expected from the discussion in Section \ref{sec:entry-results}, the bounds become very loose as $M$ increases.


In theory, we would like to have \emph{per-instance certificates-of-optimality} of the computed solution, where the lower bound is equal to the true sparsity $\|\xx_0\|_0$.  Nonetheless, we note that this ability to compute a per-instance non-trivial lower bound on the sparsity of the desired solution is an important step forward with respect to the previous approaches that require pre-computing optimality conditions for equivalence of solutions to the $\ell_0$-norm and $\ell_1$-norm minimization problems. 

We have performed a similar experiment for the group sparsity case, and observed that  the bidual framework is able to provide non-trivial lower bounds for the group sparsity also.

\begin{figure}[!t]
 \centering
\subfigure[Boxplot for $M=M_0$]{\includegraphics[width=0.32\columnwidth]{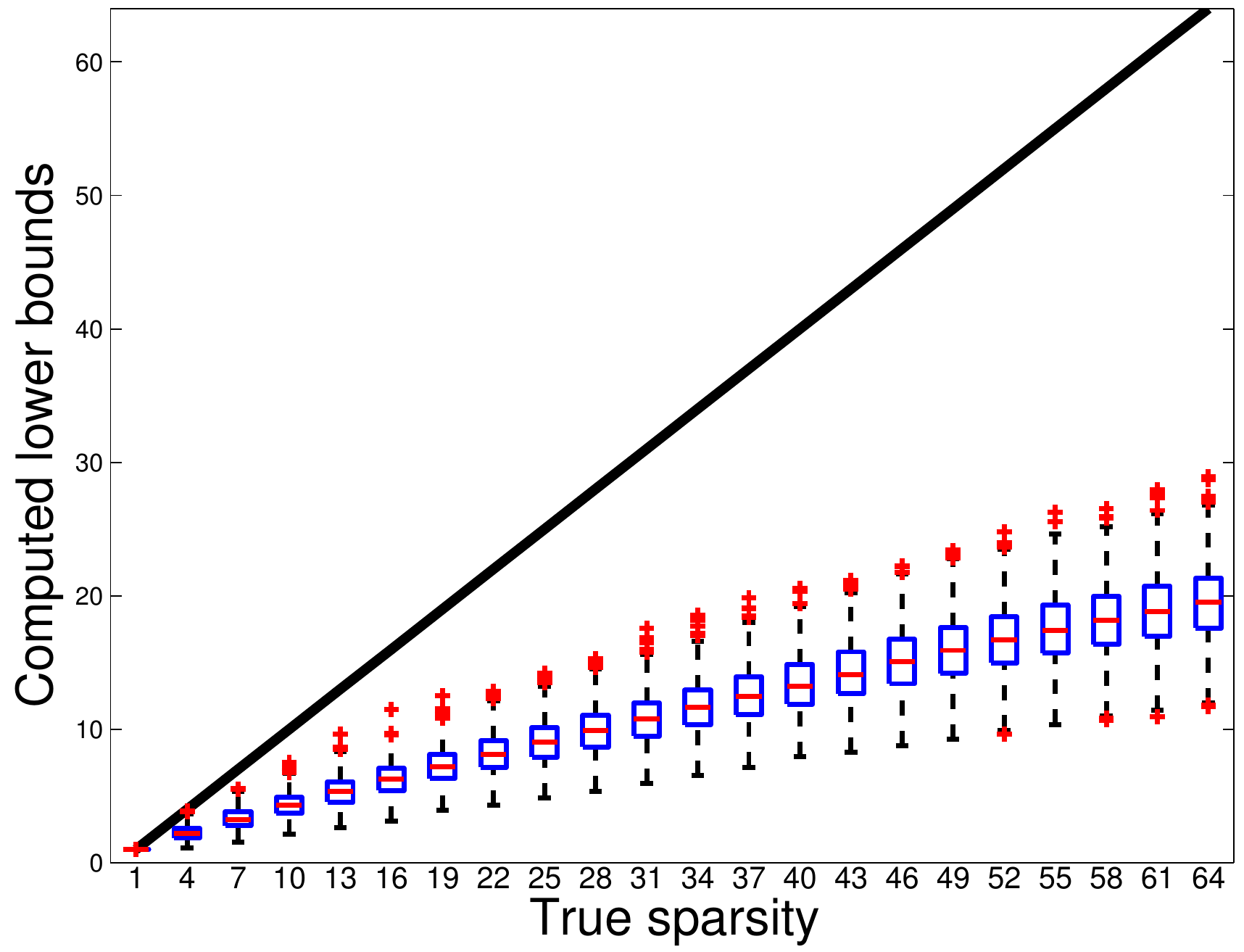} \label{fig:m0}}
\subfigure[Boxplot for $M=2M_0$]{\includegraphics[width=0.32\columnwidth]{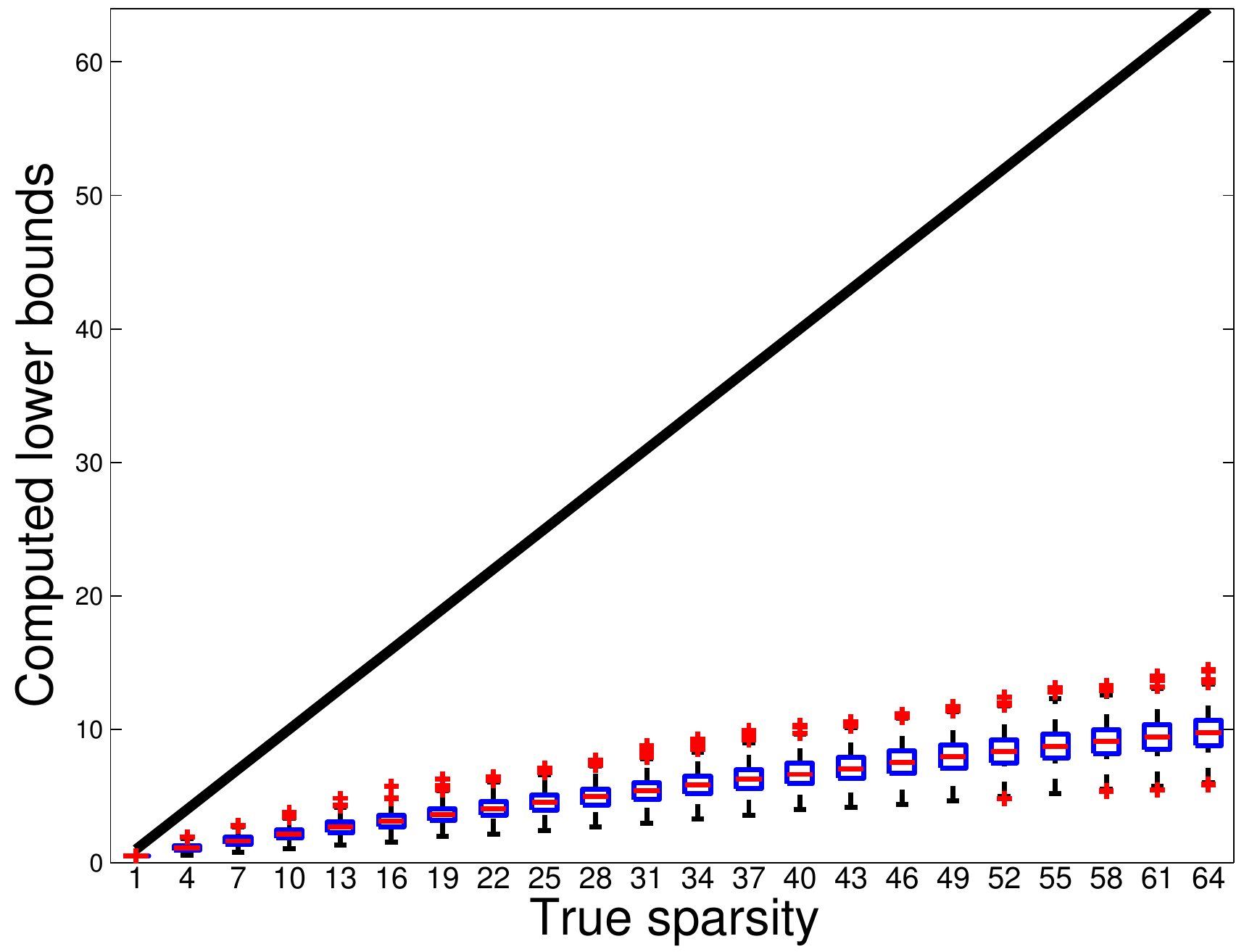} \label{fig:2m0}}
\subfigure[Boxplot for $M=5M_0$]{\includegraphics[width=0.32\columnwidth]{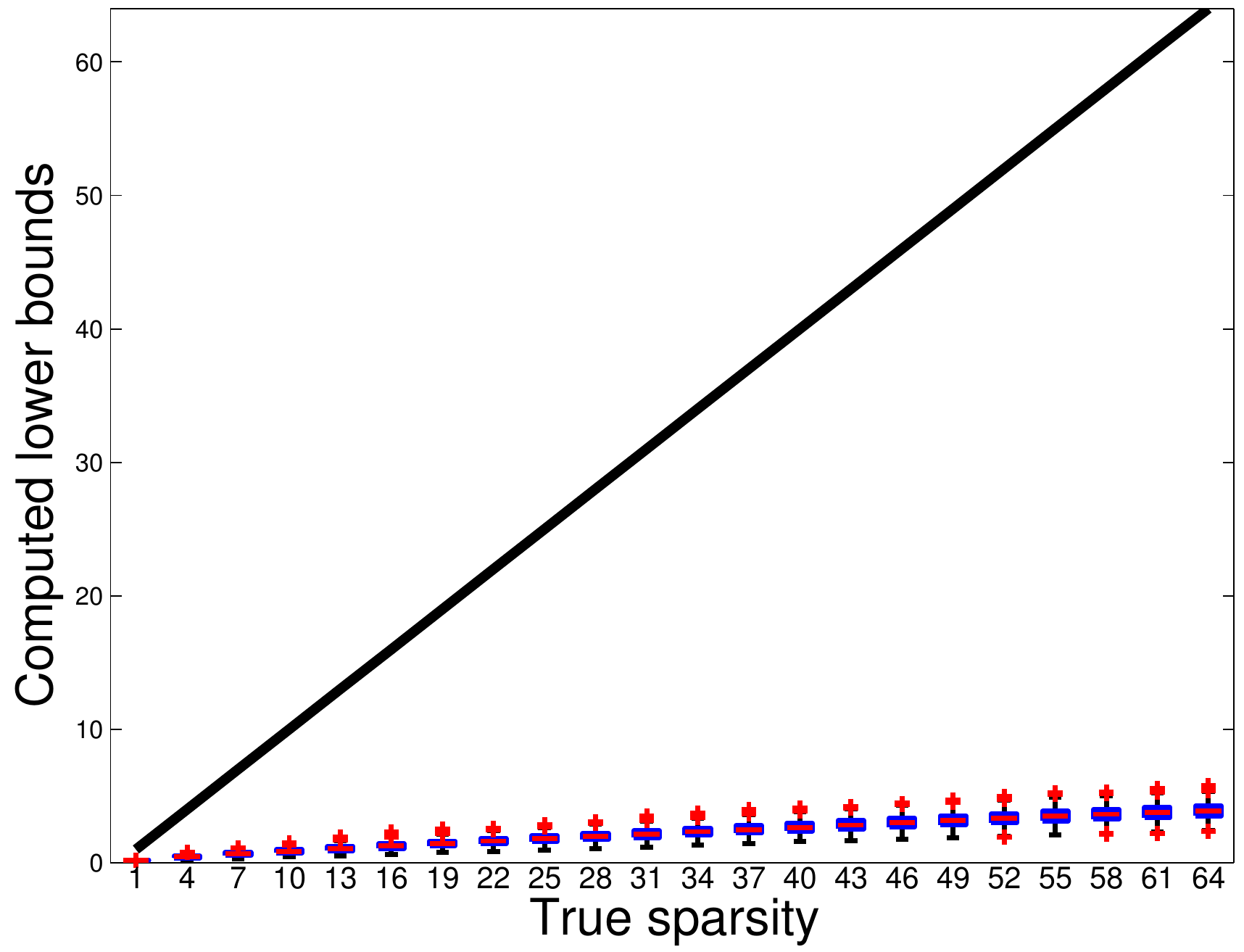} \label{fig:5m0}}\\
\caption{Results for computing the lower bounds on the true (black lines) entry-wise sparsity $\|\xx_0\|_0$  obtained over 1000 trials. 
The bounds are computed by solving \eqref{eq:entry-bidual} and using Corollary \ref{corr:l0l1} with $M=M_0, 2M_0$ and  $5M_0$, where $M_0=\|\xx_0\|_\infty$. Notice that as expected from the discussion in Section \ref{sec:entry-results}, the bounds are not tight due to the duality gap and become looser as $M$ increases.  }
  \label{fig:entry}
\end{figure}

\mypar{Mixed sparsity.} We now evaluate the results of mixed sparsity minimization for the sparsity-based face recognition problem,  where the columns of $A$ represent training images from the $K$ face classes: $A_1, \cdots, A_{K}$ and $\bb \in \Re^m$ represents a query image. We assume that a subset of pixel values in the query image may be corrupted or disguised. Hence, the error in the image space is modeled by a sparse error term $\ee$: $\bb = \bb_0 + \ee$, where $\bb_0$ is the uncorrupted image. A linear representation of the query image forms the following linear system of equations:
\begin{equation}
\bb = A\xx + \ee = \begin{bmatrix}A_1 & \cdots & A_K & I\end{bmatrix} \begin{bmatrix}\xx^\top_1& \cdots & \xx^\top_K & \ee^\top \end{bmatrix}^\top,
\end{equation}
where $I$ is the $m \times m$ identity matrix. The goal of sparsity-based classification (SBC) is to minimize the group sparsity in $\xx$ and the sparsity of  $\ee$ such that the dominant non-zero coefficients in $\xx$ reveal the membership of the ground-truth observation $\bb_0 = \bb - \ee$ \cite{WrightJ2009, ElhamifarE2011}. In our experiments, we solve for $\xx$ and $\ee$ by solving the following optimization problem:
\begin{equation}
\label{eq:mixed-res}
\{\xx_\text{1,p}^*,\ee^*_1\} =  \underset{\{\xx,\ee\}}{\arg\!\min} \sum_{k=1}^K  \|\xx_k\|_p + \gamma \|\ee\|_1  \quad \subjto \quad  A\xx + \ee= \bb.
\end{equation}

Notice that for $p=\infty$, this reduces to solving a special case of the problem in \eqref{eq:bidual-unconst}, \ie the bidual relaxation of the mixed sparsity problem with a conservative estimate of $M$. In our experiments, we set $\gamma = 0.01$ and compare the solutions to \eqref{eq:mixed-res} obtained using $p=2$ and $p=\infty$.

We evaluate the algorithms on a subset of the AR dataset \cite{AR} which has manually aligned frontal face images of size $83 \times 60$ for 50 male and 50 female subjects, \ie $K = 100$ and $m=4980$. Each individual contributes 7 un-occluded training images, 7 un-occluded testing images and 12 occluded testing images. Hence, we have 700 training images and 1900 testing images. To compute the number of non-zero blocks in the coefficient $\xx$ estimated for a testing image, we find the number of blocks whose energy $\ell_2(\xx_k)$ is greater than a specified threshold.

The results of our experiments are presented  in Figure \ref{fig:mixed}. The solution obtained with $p=2$ gives better group sparsity of $\xx$. However,  a sparser error $\ee$ is estimated with $p=\infty$.   
The number of non-zero entities in a solution to \eqref{eq:mixed-res}, \ie the number of non-zero blocks plus the number of non-zero error entries,  is lower for the solution obtained using $p = \infty$ rather than that obtained using $p=2$. However, the primal mixed-sparsity objective value  $\ell_{0,p}(\xx) + \gamma \|\ee\|_0$ (see \eqref{eq:mixed}) is lower for the solution obtained using $p=2$.


\begin{figure}[!t]
 \centering
\subfigure[Group sparsity for $p=2$]{\includegraphics[width=0.31\columnwidth]{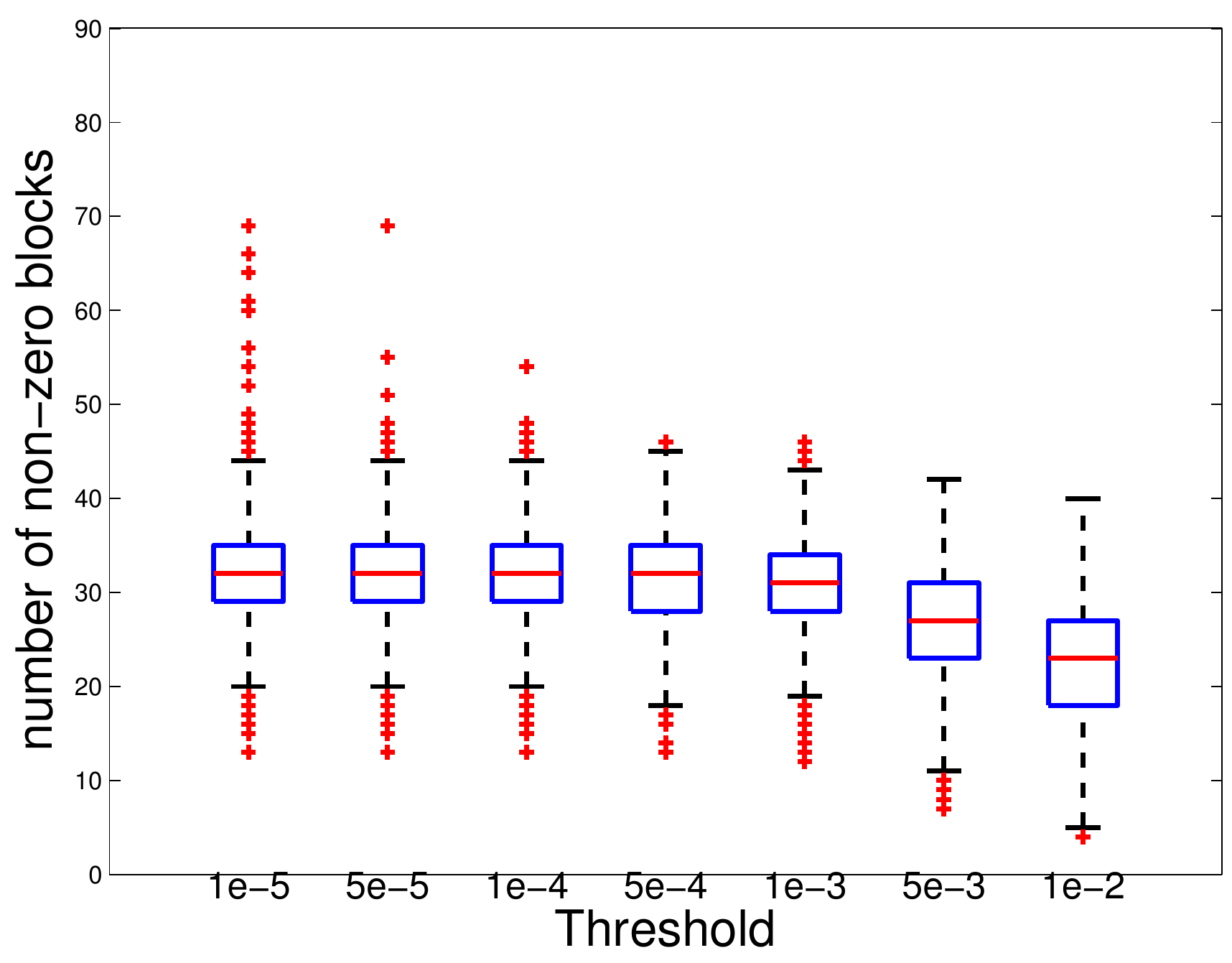}}
\subfigure[Group sparsity for $p=\infty$]{\includegraphics[width=0.31\columnwidth]{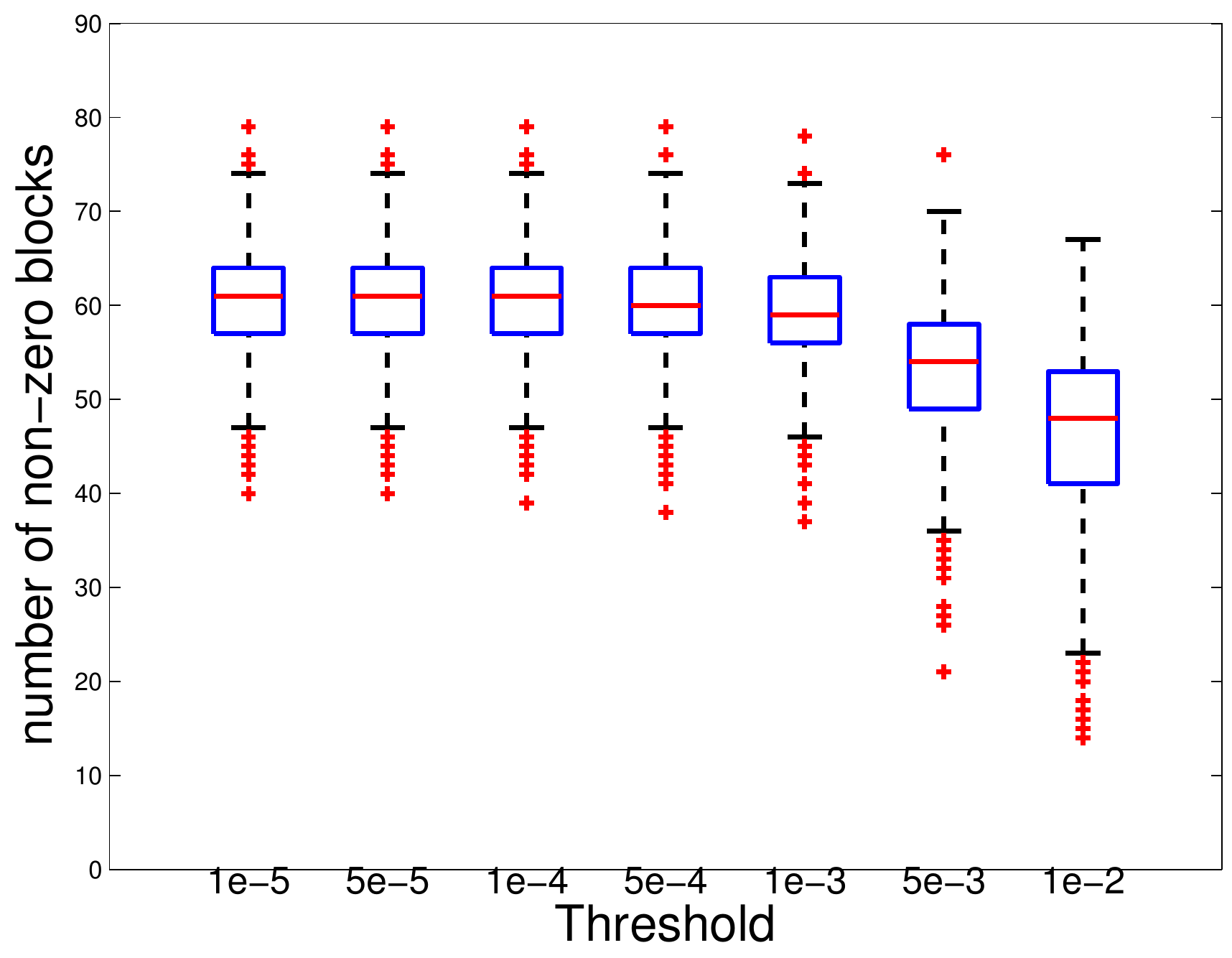} }
\subfigure[Difference in group sparsities]{\includegraphics[width=0.31\columnwidth]{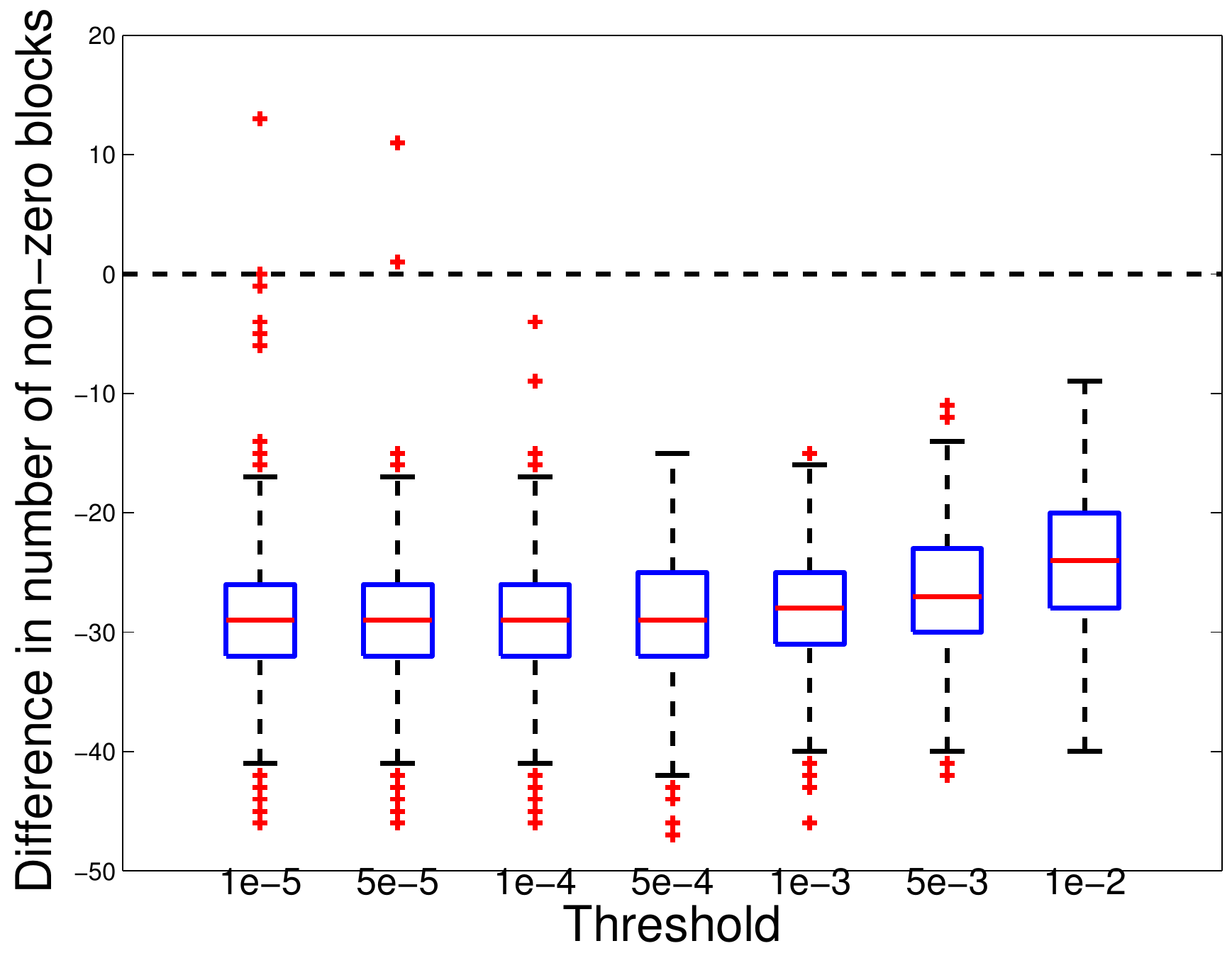} \label{fig:mixed-diff-x} }\\
\subfigure[Entry-wise sparsity for $p=2$]{\includegraphics[width=0.31\columnwidth]{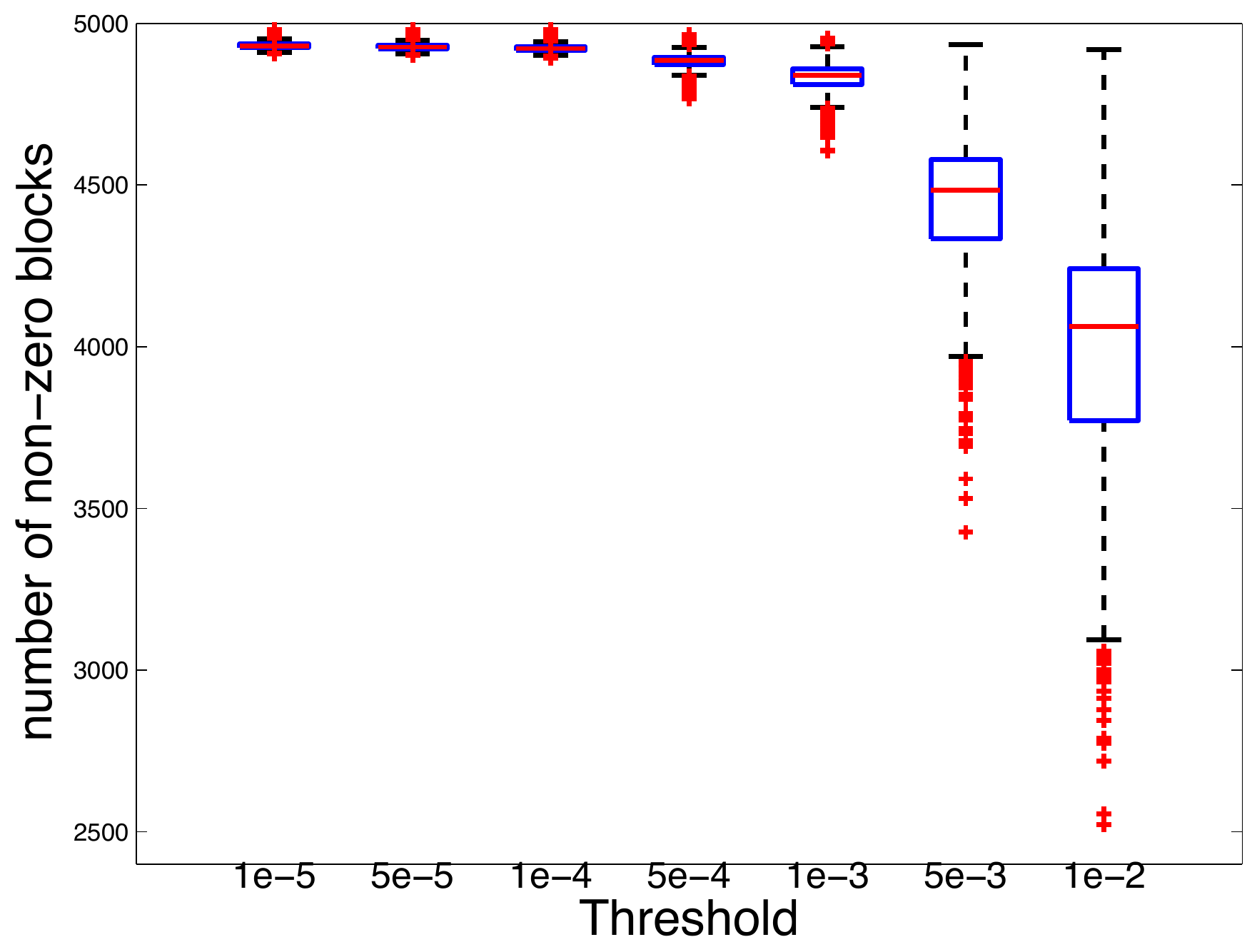}}
\subfigure[Entry-wise sparsity for $p=\infty$]{\includegraphics[width=0.31\columnwidth]{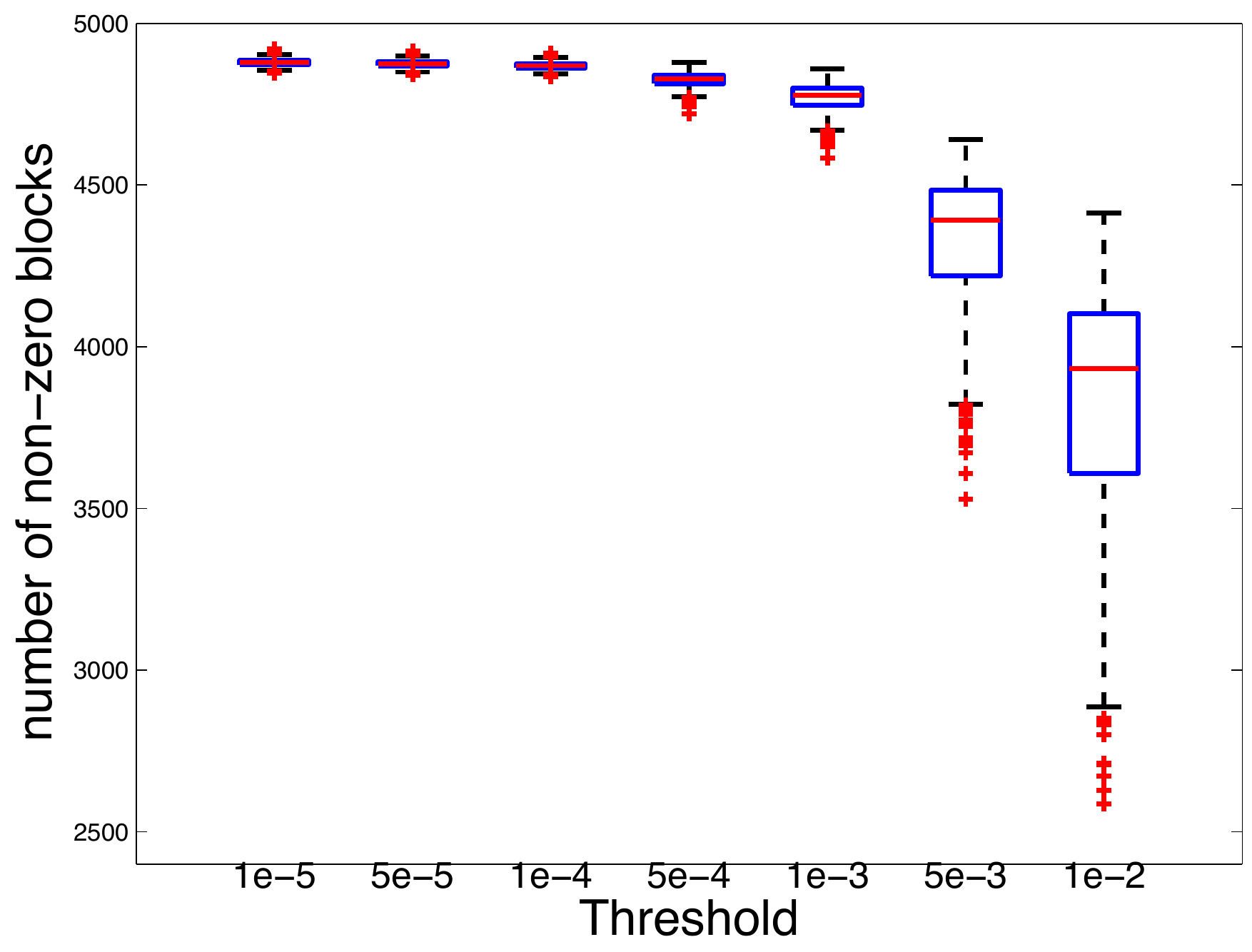} }
\subfigure[Diff. in entry-wise sparsities]{\includegraphics[width=0.31\columnwidth]{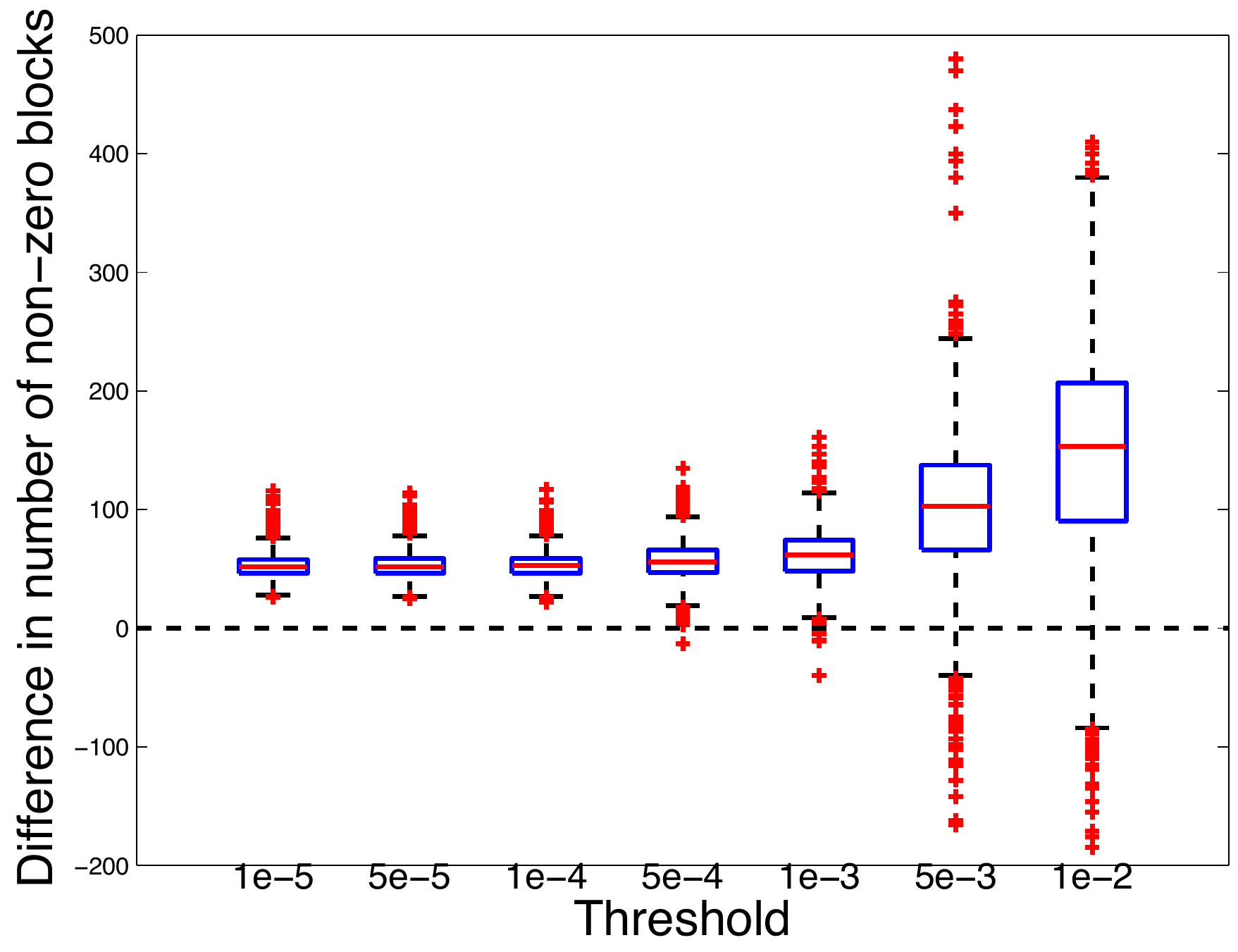}  \label{fig:mixed-diff-e}}\\
\caption{Comparison of mixed sparsity of the solutions to \eqref{eq:mixed-res} for $p=2$ and $p=\infty$. We present boxplots for group sparsity of $\xx$ and  entry-wise sparsity of $\ee$. The differences are calculated as  ($\#$ non-zero blocks/elements for $p=2$) - ($\#$ non-zero blocks/elements for $p=\infty$). We see that for $p=2$ we get better group sparsity of $\xx$, but we get a more sparse error $\ee$ when we use $p=\infty$.} 
  \label{fig:mixed}
\end{figure}

We now compare the classification results obtained with the solutions $\xx$ computed in our experiments.  For classification, we consider  the non-zero blocks in  $\xx$ and then assign the query image to the block, \ie subject class, for which it gives the least  $\ell_2$ residual $\|\bb -A_k\xx_k\|_2$. The results are presented in Table \ref{tab:mixed}. Notice that the classification results obtained with $p=\infty$ (the bidual relaxation) are better than those obtained using $p=2$. Since the classification of un-occluded images is already very good using $p=2$,  classification with $p=\infty$ gives only a minor improvement in this case. However, a more tangible improvement is noticed in the classification of the occluded images. Therefore the classification with $p=\infty$ is in general better than that obtained with $p=2$, which is considered the state-of-the-art for sparsity-based classification \cite{ElhamifarE2011}.

\begin{table}
\centering
\begin{tabular}{ | c | c | c | c | c|} 
\hline
& \multicolumn{2}{c|}{$p=2$}  & \multicolumn{2}{c|}{$p=\infty$}\\ \hline
& $\#$(correct results) & $\%$(correct results)& $\#$(correct results)& $\%$(correct results)\\ \hline
un-occluded & $655$ & $93.57\%$  & $663$ &  $94.71\%$\\
occluded &  $643$ & $53.58\%$  & $691$ &  $57.58\%$\\
total & 1298 & $68.32\%$ & $1324$ & $69.68\%$ \\ \hline
\end{tabular}
\caption{Classification results on the AR dataset using the solutions obtained by minimizing mixed sparsity. The test set consists of 700 un-occluded images and 1200 occluded images.}
\label{tab:mixed}
\end{table}

\section{Discussion}
We have presented a novel analysis of several sparsity minimization
problems which allows us to interpret several convex relaxations of
the original NP-hard primal problems as being equivalent to maximizing
their Lagrangian duals. The pivotal point of this analysis is the
formulation of mixed-integer programs which are equivalent to the
original primal problems.  While we have derived the biduals for only
a few sparsity minimization problems, the same techniques can also be
used to easily derive convex relaxations for other sparsity
minimization problems \cite{CevherV2008}.

An interesting result of our biduality framework is the ability to
compute a per-instance certificate of optimality by providing a lower
bound for the primal objective function. This is in contrast to most
previous research which aims to characterize either the subset of
solutions  or the set of conditions for perfect sparsity recovery
using the convex relaxations
\cite{CandesE2005,CandesE2008,DonohoD2003,DonohoD2004,DonohoD2006,FletcherA2008,Iouditski:arxiv,ReevesG2009}. In
most cases, the  conditions are either weak or hard to verify. More
importantly, these conditions needed to be pre-computed as opposed to
verifying the correctness of a solution at run-time. In lieu of this,
we hope that our proposed framework will prove an important step
towards per-instance verification of the solutions. Specifically, it is of interest in the future to explore tighter relaxations for the verification of the solutions.

\subsubsection*{Acknowledgments}
This research was supported in part by  ARO MURI W911NF-06-1-0076, ARL MAST-CTA W911NF-08-2-0004, NSF CNS-0931805, NSF CNS-0941463 and NSF grant 0834470.
The views and conclusions contained in this document are those of the authors and should not be interpreted as representing the official policies, either expressed or implied, of the Army Research Laboratory or the U.S. Government. The U.S. Government is authorized to reproduce and distribute for Government purposes notwithstanding any copyright notation herein.

\end{document}